\documentclass{article}

\usepackage{arxiv}

\usepackage{natbib}

\usepackage{microtype}
\usepackage{graphicx}
\usepackage{subfigure}
\usepackage{booktabs} 

\usepackage[normalem]{ulem}  

\usepackage{hyperref}

\usepackage[normalem]{ulem}

\usepackage[textsize=tiny]{todonotes}

\usepackage{amsmath}
\usepackage{amssymb}
\usepackage{mathtools}
\usepackage{amsthm}

\newcommand{\gril}{\textsc{Gril}}
\newcommand{\grild}{\textsc{D-Gril}}
\newcommand{\RKG}{\mathsf{rk}}
\newcommand{\RR}{\mathbb{R}}
\newcommand{\comp}[2]{\mathbf{#1} \leq \mathbf{#2}}

\newcommand{\cG}{\mathcal{G}}

\newcommand{\vp}{\mathbf{p}}
\newcommand{\cH}{\mathcal{H}}
\usepackage[capitalize,noabbrev]{cleveref}

\theoremstyle{plain}
\newtheorem{theorem}{Theorem}[section]
\newtheorem{proposition}[theorem]{Proposition}

\newtheorem{corollary}[theorem]{Corollary}
\theoremstyle{definition}
\newtheorem{definition}[theorem]{Definition}

\theoremstyle{definition}
\newtheorem{remark}[theorem]{Remark}
\theoremstyle{definition}

\definecolor{azure}{rgb}{0.0, 0.5, 1.0}

\newcommand{\cancel}[1]

\title{\textsc{D-Gril}: End-to-End Topological Learning with \texorpdfstring{$2$}{}-parameter Persistence}

\author{%
  Soham Mukherjee* \\
  Department of Computer Science \\
  Purdue University \\
  West Lafayette, IN \\
  \texttt{mukher26@purdue.edu}
  \And
  Shreyas N. Samaga* \\
  Department of Computer Science\\
  Purdue University\\
  West Lafayette, IN \\
  \texttt{ssamaga@purdue.edu} \\
  \And
   Cheng Xin \\
   Department of Computer Science \\
   Rutgers University \\
   New Brunswick, NJ \\
   \texttt{cx122@cs.rutgers.edu}
   \And
   Steve Oudot \\
   GeomeriX team \\
   Inria Saclay and \'Ecole polytechnique \\
   Palaiseau, France \\
   \texttt{steve.oudot@inria.fr}
   \And
   Tamal K. Dey \\
   Department of Computer Science \\
   Purdue University \\
   West Lafayette, IN \\
   \texttt{tamaldey@purdue.edu}
}

\begin{document}

\maketitle
\def\thefootnote{*}\footnotetext{Both are considered first authors 
}\def\thefootnote{\arabic{footnote}}
\begin{abstract}
End-to-end topological learning using $1$-parameter persistence is well known. We show that the framework can be enhanced using $2$-parameter persistence by adopting a recently introduced $2$-parameter persistence based vectorization technique called \gril{}. We establish a theoretical foundation of \emph{differentiating} \gril{} producing \grild{}. We show that \grild{} can be used to learn a bifiltration function on standard benchmark graph datasets. Further, we exhibit that this framework can be applied in the context of bio-activity prediction in drug discovery.   
\end{abstract}

\section{Introduction}
\label{sec:introduction}
In recent years, persistent homology, one of the flagship concepts of Topological Data Analysis (TDA), has found its use in many fields such as neuroscience, material science, sensor networks, shape recognition, gene expression data analysis and many more~\cite{donut,tda_guide}. The performance of machine learning models such as Graph Neural Networks (GNNs) can be enhanced by augmenting topological information captured by persistent homology~\citep{Hofer2017DeepLW,GNN_graph_topology19,PersLay,togl}. Classical persistent homology, also known as $1$-parameter persistence, captures the evolution of topological structures in a simplicial complex $K$ following a filter function $f \colon K \to \RR$. The evolution of topological structures, in this case, can be completely characterized and compactly represented as \emph{persistence diagrams} or equivalently \emph{barcodes}~\cite{computing_pers_hom, edelsbrunner2002topological}. These persistence diagrams or barcodes can be vectorized~\cite{bubenik2015perslandscapes,reininghaus2015stable,AdamPersImg,hofer2019learning,PersLay,PLLay} and used in machine learning pipelines. In most applications, the simplicial complex $K$ is given and the choice of the filter function $f$ depends on the application. Choosing an appropriate filter function can be challenging. To avoid this, in \cite{hofer2020graph}, the authors proposed an end-to-end learning framework to learn the filter function rather than relying on making the right choice. They showed that learning the filter function performs better than the standard choices of filter functions on many graph datasets.

To obtain richer topological information, we can consider an $\RR^n$-valued filter function instead of a scalar filter function. This leads to the structure of a multiparameter persistence module instead of $1$-parameter persistence module. Multiparameter persistence modules, unlike their $1$-parameter counterparts, can \emph{not} be classified completely using a discrete invariant~\cite{carlsson2009computing}. Consequently, vectorizing multiparameter persistence modules while retaining as much topological information as possible has become a challenging problem. There are many recent works in this direction which use incomplete invariants such as \emph{rank invariant}~\cite{carlsson2009computing} or equivalently \emph{fibered barcodes}~\cite{rivet}, multiparameter persistence images~\cite{Carriere_Multipers_Images}, multiparameter persistence landscapes~\cite{Multipers_landscapes}, multiparameter persistence kernel~\cite{Multipers_Kernel_Kerber}, vectorization of signed barcodes~\cite{vect_signed_barcodes_23}, topological fingerprinting for virtual screening using multidimensional persistence~\cite{Todd}, effective multidimensional persistence~\cite{emp} or \emph{generalized rank invariant}~\cite{gril23}. The authors, in all these works, show that $2$-parameter persistence methods perform better than $1$-parameter persistence methods on many graph and time-series datasets, suggesting that an $\RR^2$-valued filter function, indeed, captures richer topological information than a scalar filter function. However, in all these works, one needs to make a suitable choice for an $\RR^2$-valued filter function. Akin to the $1$-parameter setup, learning the filter function rather than choosing a filter function can prove to be more informative and beneficial for the task at hand.

In this paper, we propose an end-to-end learning framework using $2$-parameter persistent homology based on the vectorization \gril{} introduced by~\cite{gril23}. We show that \gril{} is piecewise affine and give an explicit formula for the differential of \gril. The results discussed in~\cite{davis_stochastic, carriere_stochastic} help us to show the convergence of stochastic sub-gradient descent on \gril{}. We use these results to build a differentiable topological layer \grild, consequently an end-to-end learning pipeline. We use \grild{} for bio-activity prediction and provide experimental results. One of the key advantages of using TDA in the context of bio-activity prediction is its ability to capture and analyze the shape and structural characteristics of molecules in a manner that traditional methods might overlook. Molecules possess intricate shapes and topological features that directly influence their interactions with biological targets. TDA can uncover these subtleties, allowing for a deeper understanding of structure-activity relationships. We also show that \grild{} can be used, more generally, for bifiltration learning on graphs and provide experimental results where we compare with existing multiparameter persistent homology methods on benchmark graph datasets.
\footnote{The complete code is available at \url{https://github.com/TDA-Jyamiti/d-gril}}

In a recent work, \cite{diff_signed_barcodes_24} present an end-to-end topological learning method by developing a general framework suitable for a class of invariants of multiparameter persistence. Our approach though specific to \gril{} is more direct, uses simpler arguments, and is faster for training.

\section{Overview}
\label{sec:overview}

\cancel{
In~\cite{gril23}, the authors introduced a 2-parameter persistence based vectorization called \gril. \gril{} (Generalized Rank Invariant Landscape) is based on an invariant of 2-parameter persistence modules (formally defined in section \ref{sec:background}) called Generalized Rank (formally defined in \ref{sec:background}). Due to a result in~\cite{DKM24}, generalized rank for 2-parameter persistence modules can be computed efficiently. \gril{} is defined in the same spirit as persistence landscapes~\cite{bubenik2015perslandscapes} where the traditional rank function is replaced by generalized rank for 2-parameter persistence modules. The authors consider special form of intervals in $\RR^2$ called $\ell$-worms and compute generalized ranks over these intervals. An $\ell$-worm is an interval parametrized by a center point $\vp$ and a width $d$. An interesting property of generalized ranks over these intervals is that the values of generalized rank monotonically decrease when the intervals get larger.
Thus, given a value of rank $k$, the worm is expanded until the value of generalized rank over the worm drops below $k$. The maximum value of the width, $d_{\vp}$, of the worm where the value of generalized rank is greater than or equal to $k$ is the value of \gril{} at that center point for a given rank value $k$ and the value of $\ell$ in $\ell$-worm. This value $d_{\vp}$ constitutes an entry in the \gril{} vector. The center points are sampled from a finite grid and each center point $\vp_i$ contributes a $d_{\vp_i}$ to the final \gril{} vector. This is a brief overview of \gril{} vectorization. We refer the reader to~\cite{gril23} for a detailed description and an algorithm to compute \gril{}. 
}

Our main construct is \gril{}, introduced in~\cite{gril23} for vectorizing a $2$-parameter persistence module. \gril{} is defined in the same spirit as persistence landscapes~\cite{bubenik2015perslandscapes}, introduced for $1$-parameter persistence. The landscape uses the rank function over the interval decomposition known for $1$-parameter persistence modules. To extend the concept to $2$-parameter persistence modules, one faces two main challenges: (i) in general, persistence modules beyond $1$-parameter may not
decompose into intervals, (ii) the usual rank function cannot be defined
for intervals that are not rectangular. \gril{} solves (i) by defining the landscape function over a fixed set of two-dimensional intervals called \emph{$\ell$-worms} instead of the support of the indecomposables of the modules, and solves (ii) by considering a generalized version of the rank function termed as \emph{generalized rank} (Definition \ref{def:gen_rank}). 
%
 An $\ell$-worm is an interval in $\mathbb{R}^2$ parameterized by a center point $\vp$, a length~$\ell$, and a width $d$. In practice, a discrete version of the $\ell$-worms is used (Definition~\ref{def:discrete-l-worm} and Figure~\ref{fig:constraining_coord}), which allows for the efficient computation of the generalized rank over them using zigzag persistence~\cite{DKM24}.

An interesting property of the generalized rank over intervals is that it monotonically decreases as the intervals get larger.
Thus, given fixed parameters $k>0$ and $\ell>0$, the width~$d$  of an $\ell$-worm can be increased until the generalized rank over the worm drops below $k$. The maximum value of the width thus obtained, $d_{\vp}$, is the value of \gril{} at the worm center point $\vp$ for $k$ and $\ell$. 
In practice, $s$ center points $\vp_1, \cdots, \vp_s$ are sampled from the plane, and each center point $\vp_i$ contributes an entry $d_{\vp_i}$ in the final \gril{} vector. We refer the reader to~\cite{gril23} for a detailed description and an algorithm to compute \gril{}. 

\gril{} yields a vector representation of the 2-parameter persistence module formed from a given simplicial complex and a bifiltration function (formally defined in section~\ref{sec:background}). However, there is no provision to learn this bifiltration function. 
Often, it is difficult to choose the right filtration function that elicits the relevant information present in the data. For $1$-parameter persistence, learnt filter functions have been shown to perform better than the filter functions popularly chosen in TDA. 
To learn a suitable bifiltration function, we build a differentiable layer called \grild{}. We show in section~\ref{sec:experiments} that the learnt bifiltration function performs better than some of the popular bifiltration functions known in TDA supporting the need for end-to-end learning.

Towards this goal, we study the differentiability of \gril{}, seeking a closed form equation for its differential wherever defined. This, in turn, enables us to differentiate through the \gril{} computation, which is necessary for the topological layer in our end-to-end learning framework. In this regard, we draw upon the results discussed in~\cite{carriere_stochastic, davis_stochastic} for the conditions required for convergence of stochastic sub-gradient descent. We show that our framework satisfies these conditions and, consequently, that stochastic sub-gradient descent converges almost surely.



\section{Background}
\label{sec:background}
We review and recall some definitions and results in this section. Section \ref{subsec:bg_mph} primarily consists of concepts in multiparameter persistent homology. Section \ref{subsec:stoc_subgrad} consists of concepts and results that establish the conditions required for the convergence of stochastic sub-gradient descent. Section \ref{subsec:bg_ominimal} contains some background which is required to prove that \gril{} satisfies these conditions.

\subsection{Multiparameter persistent homology}\label{subsec:bg_mph}
In this subsection, we briefly review the concepts in multiparameter persistent homology. We focus, primarily, on 2-parameter persistent homology. For detailed definitions, refer~\cite{edelsbrunner2010computational,dey_wang_2022_book}. We begin by recalling the definition of a simplicial complex. Given a finite vertex set $V$, a simplicial complex $K=K(V)$ defined on this vertex set is a collection of subsets of $V$ such that if a subset $\sigma\subseteq V$ is in $K$, then all proper subsets
$\tau\subset \sigma$ are also in $K$. Each element in $K$ with cardinality $k+1$ is called a $k$-simplex or simply
a simplex. A graph is a simplicial complex where $V$ is the vertex set of the graph and $K$ consists of edges and vertices of the graph. A \emph{filtration}
is a collection of simplicial complexes $\{K_\mathbf{x}\}_{\mathbf{x}\in \mathbb{R}}$ indexed by the reals, with the property that $K_\mathbf{x}\subseteq K_{\mathbf{y}}$ for each
$\mathbf{x}\leq \mathbf{y}$. Extending this definition to collections of complexes indexed by $\RR^2$ with the product partial order ($\mathbf{x} \leq \mathbf{y} \in \RR^2$ if $x_1 \leq y_1$ and $x_2 \leq y_2$), we get
a \emph{bifiltration}. 

\begin{definition}[Bifiltration]
    A bifiltration is a collection of simplicial complexes $\{K_\mathbf{x}\}_{\mathbf{x} \in \RR^2}$ where $K_{\mathbf{x}} \subseteq K_{\mathbf{y}}$ for all comparable $\comp{x}{y} \in \RR^2$.
\end{definition}

    Let $K$ be a simplicial complex and $f \colon K \to \RR^2$ be a map with the property that $f(\sigma) \leq f(\tau) \in \RR^2$ for all $\sigma \subseteq \tau$. For each $\mathbf{x} \in \RR^2$, let $K_\mathbf{x} \coloneqq \{ \sigma \in K \colon f(\sigma) \leq \mathbf{x} \}$. Clearly, $K_\mathbf{x} \subseteq K_{\mathbf{y}}$ for all comparable $\comp{x}{y} \in \RR^2$. The resulting bifiltration, denoted
    $\{K^f_{\mathbf{x}}\}_{\mathbf{x}\in \RR^2}$, is called the \emph{sub-level set} bifiltration of the \emph{bifiltration function} $f$.


A filtration in the $1$-parameter case induces a persistence module, obtained by considering the inclusion-induced linear maps between the vector spaces given by the homology groups of the simplicial complexes comprising the filtration. Similarly, we obtain a $2$-parameter persistence module from a bifiltration.

\begin{definition}[2-parameter persistence module]
    A 2-parameter persistence module is an assignment of finite-dimensional vector spaces $M_{\mathbf{x}}$ for each $\mathbf{x} \in \RR^2$, and of linear maps $M_{\comp{x}{y}} \colon M_\mathbf{x} \to M_\mathbf{y}$ for all comparable $\comp{x}{y} \in \RR^2$, with the properties that $M_{\comp{x}{x}} = \text{id}$ and $M_{\comp{y}{z}} \circ M_{\comp{x}{y}} = M_{\comp{x}{z}}$ for all $\comp{\comp{x}{y}}{z} \in \RR^2$.
\end{definition}
Readers familiar with category theory can recognize that a 2-parameter persistence module is a functor $M \colon \RR^2 \to \textbf{vec}_{\mathbb{F}}$ where the poset $\RR^2$ is regarded as a category and $\textbf{vec}_{\mathbb{F}}$ denotes the category of finite-dimensional vector spaces over a field $\mathbb{F}$.

    Given a bifiltration $\{K_\mathbf{x}\}_{\mathbf{x}\in \RR^2}$, we consider $H_i(K_\mathbf{x})$, the $i$th homology group of the simplicial complex $K_{\mathbf{x}}$ over a field $F$, say $\mathbb{Z}_2$. Then, for each inclusion $K_{\mathbf{x}} \subseteq K_{\mathbf{y}}$, we get an induced linear map $H_i(K_\mathbf{x}) \to H_i(K_\mathbf{y})$. By this assignment, we get a 2-parameter persistence module.

    Let $\{K^f_{\mathbf{x}}\}_{\mathbf{x} \in \RR^2}$ be a sub-level set bifiltration. Then the 2-parameter persistence module $M$ obtained by the homology assignment is called the 2-parameter persistence module induced by $f$, and we denote it as $M_f$.

As mentioned in Section~\ref{sec:overview}, \gril{} vectorization uses the concept of generalized rank introduced by~\cite{GenRankKim21}, which we define below formally.

\begin{definition}[Generalized Rank]\label{def:gen_rank}
    Let $M$ be a 2-parameter persistence module. The restriction of $M$ to an interval $I$ of $\RR^2$ (see Appendix \ref{app:diff_gril}), denoted by $M|_I$, is the diagram formed by the collection of vector spaces $M_\mathbf{x}$ for $\mathbf{x} \in I$ and linear maps $M_{\comp{x}{y}}$ for all comparable $\comp{x}{y} \in I$. Then, the generalized rank of $M$ over $I$ is defined as the rank of the canonical map from the limit $\varprojlim M|_I$ to colimit $\varinjlim M|_I$:
    \begin{equation*}
        \RKG^M(I) \coloneqq \text{rank}\left(\varprojlim M|_I \rightarrow     \varinjlim M|_I\right).
    \end{equation*}
\end{definition}
 We refer the reader to~\cite{Saunders_Maclane_Cat_Theory} for the definitions of limit, colimit, and the construction of the canonical limit-to-colimit map. Intuitively, generalized rank captures the number of independent topological features supported over the interval $I$. It can be computed efficiently for $2$-parameter persistence modules~\cite{DKM24} using zigzag persistence~\cite{dh22}.
\begin{remark}
    In the special case where $I$ is a rectangle, $\RKG^M(I)$ is the rank of the linear map $M_{\mathbf{u}\leq \mathbf{v}}$ from the lower left corner $\mathbf{u}$ of the rectangle to the upper right corner $\mathbf{v}$. 
\end{remark}

\subsection{Stochastic sub-gradient descent}
\label{subsec:stoc_subgrad}
In this subsection, we briefly review the result and the conditions under which stochastic sub-gradient descent  converges as shown in~\cite{davis_stochastic}. In ~\cite{carriere_stochastic}, authors use this result in the setting of $1$-parameter persistent homology and show that stochastic sub-gradient descent converges in that case. 

Given a loss function $\mathcal{L}$ with the objective of minimizing it, 
consider the differential inclusion 
\begin{equation*}
   \Dot{\mathbf{z}}(t) \in -\partial \mathcal{L}(\mathbf{z}(t)) \hspace{0.3cm} \text{ for almost every } t. 
\end{equation*}
 The solutions $\mathbf{z}(t)$ are the trajectories of the sub-gradient of $\mathcal{L}$ which can be approximated by the standard stochastic sub-gradient descent method given by: 
\begin{equation}
\label{eq:stoc_subgrad}
    \mathbf{x}_{k+1} = \mathbf{x}_k - \alpha_k(\mathbf{y}_k + \mathbf{\zeta}_k), \hspace{0.3cm} \mathbf{y}_k \in \partial \mathcal{L}(\mathbf{x}_k),
\end{equation}
where the sequence $(\alpha_k)_k$ is the learning rate and $(\zeta_k)_k$ is a sequence of random variables or ``noise". \cite{davis_stochastic}
establish that under mild technical assumptions (Assumption C in \cite{davis_stochastic}) on these two sequences, the stochastic sub-gradient method converges almost surely to critical points of $\mathcal{L}$ if $\mathcal{L}$ is locally Lipschitz and Whitney stratifiable.
\begin{proposition}[Corollary 5.9~\cite{davis_stochastic}]\label{prop:sgd_convergence}
    Let $f\colon \RR^d \to \RR$ be a locally Lipschitz function that is $C^d$-stratifiable. Consider the iterates $\{ \mathbf{x}_k\}_{k\geq 1}$ produced by the stochastic sub-gradient method (Eq~\eqref{eq:stoc_subgrad}) and suppose Assumption C of \cite{davis_stochastic} holds. Then, almost surely, every limit point of the iterates $\{\mathbf{x}_k \}_{k\geq 1}$ is critical for $f$ and the function values $\{f(\mathbf{x}_k) \}_{k\geq 1}$ converge. 
\end{proposition}

\subsection{o-minimal geometry}\label{subsec:bg_ominimal}
It is known that any \emph{definable} function in an \emph{o-minimal} structure admits a Whitney $C^p$ stratification for all $p \geq 1$ \cite{ominimal}. We recall the definitions of o-minimal structure and definable function here.

\begin{definition}[o-minimal structure]
    An o-minimal structure on $\RR$ is a collection $\{ S_n\}_{n\in \mathbb{N}}$ where each $S_n$ is a set of subsets of $\RR^n$ such that:
    \begin{enumerate}
        \item $S_1$ is exactly the collection of finite union of points and intervals.
        \item $S_n$ contains all the sets of the form $\{\mathbf{x} \in \RR^n \colon p(\mathbf{x}) = 0 \}$ where $p$ is a polynomial on $\RR^n$.
        \item $S_n$ is a Boolean sub-algebra of $\RR^n$ for all $n$.
        \item If $A \in S_n$ and $B \in S_m$, then $A \times B \in S_{n+m}$.
        \item If $\pi \colon \RR^{n+1} \to \RR^n$ is the canonical projection onto the first $n$-coordinates, then for $A \in S_{n+1}$, $\pi(A) \in S_n$.
    \end{enumerate}
\end{definition}

A subset $A \in S_n$ for some $n\in \mathbb{N}$ is known as a \emph{definable set} in the o-minimal structure. Let $A \in S_n$ be given. A function $f \colon A \to \RR^m$ is said to be \emph{definable} if the graph of the function in $\RR^{n+m}$ is a definable set.

We note that all semi-algebraic functions are definable. In fact, the author in, \cite{exp_definable}, shows that there exists an o-minimal structure that simultaneously contains all semi-algebraic sets and the graph of the exponential function. The authors in \cite{davis_stochastic} use this result to show that a neural network, being a composition of definable functions, is definable (Corollary 5.11 \cite{davis_stochastic}). For readers not familiar with the notion of definable functions, it is sufficient to know that a piecewise affine map is definable. 



\section{Differentiability of \gril}
\label{sec:diff_of_gril}
We begin this section by proving that \gril{} is a piecewise affine map in section \ref{subsec:gril_piecewise_affine}. \gril{} being piecewise affine leads us to the fact that stochastic sub-gradient converges almost surely on \gril{}, which is discussed in section \ref{subsec:sgd}. In section \ref{subsec:grad_gril}, we give an explicit formula for the differential of \gril{}. 
\subsection{\gril{} as a piecewise affine map}\label{subsec:gril_piecewise_affine}
We begin by recalling the definitions of $\ell$-worm and \gril.
\begin{definition}[discrete $\ell$-worm,~\cite{gril23}]
\label{def:discrete-l-worm}
    Let $\boxed{\mathbf{p}}_d\coloneqq\{\mathbf{w}: \|\mathbf{p}-\mathbf{w}\|_\infty\leq d\}$ be 
the $d$-square centered at $\mathbf{p} \in \RR^2$ with side $2d$.
Given $ \ell\geq 1 \text{ and } d > 0$, the \emph{$\ell$-worm}, $\boxed{\mathbf{p}}_d^\ell$, is defined as the union of all $d$-squares $\boxed{\mathbf{q}}_d$ centered at some point $\mathbf{q}$ on the off-diagonal line segment $\mathbf{p}\pm\alpha\cdot(1, -1)$ with $\alpha = j \cdot d$ where $j \in \{1, \hdots, \ell-1\}$.
\end{definition}
Refer to Figure \ref{fig:constraining_coord} for a $2$-worm.
\begin{definition}[\gril,~\cite{gril23}]
    For a 2-parameter persistence module $M$, the Generalized Rank Invariant Landscape (\gril) is a function $\lambda^M:\RR^2\times\mathbb{N}_+\times\mathbb{N}_+\rightarrow \RR$ defined as 
\begin{equation*}
\lambda^M(\mathbf{p}, k, \ell)\coloneqq \sup \left\{d\geq 0 \colon \RKG^M \left (\boxed{\mathbf{p}}^\ell_d \right)\geq k\right\}. 
\end{equation*}
\label{def:gril}
\end{definition}

Fixing the bifiltration function $f$ and $k,\ell$, the landscape
$\lambda^{M}$ provides a function $\lambda^{M_f}_{k,\ell}:\mathbb{R}^2\rightarrow \mathbb{R}$, $\mathbf{\mathbf{\vp}}\mapsto \lambda^{M_f}(\mathbf{\mathbf{\vp}}, k, \ell)$.
The \gril{} vector is the vector of values of $\lambda^{M_f}_{k,\ell}$ evaluated
at a set of chosen sample points $\{\vp_1,\ldots,\vp_s\}\subset \mathbb{R}^2$.



Let $K$ be a simplicial complex with $n$ simplices, labelled $\sigma_1, \sigma_2, \hdots, \sigma_n$. A bifiltration function $f \colon K \to \RR^2$ can be viewed as a vector $\mathbf{v}_f \in \RR^{2n}$ where, for $k=1,\ldots, n$, $\mathbf{v}_f[2k-1]=f_x(\sigma_k)$ and $\mathbf{v}_f[2k]=f_y(\sigma_k)$. Here, $f_x(\sigma)$ and $f_y(\sigma)$ denote the $x$- and $y$-coordinates of the vector $f(\sigma)$ respectively. Notice that the vectors in $\RR^{2n}$ that correspond to a valid bifiltration function form a convex cone. We work with this set of vectors in $\RR^{2n}$. In this setting, the authors in~\cite{gril23} show that \gril{} is Lipschitz continuous in the following sense.

\begin{proposition}\label{prop:gril_cont}
    Let $X$ be a discrete space with $\vert X \vert = n$. For fixed $k,\ell,\vp$, let $\Lambda_{k,\ell}^\mathbf{p} \colon \RR^{2n} \to \RR$ be the map 
    $\mathbf{v}_f \mapsto \lambda^{M_f}_{k, \ell}(\mathbf{p})$.
    Then, $\Lambda^\mathbf{p}_{k,\ell}$ is Lipschitz continuous.
\end{proposition}

By Rademacher's theorem~\cite{federer2014geometric}, $\Lambda^\mathbf{p}_{k,\ell}$ is differentiable almost everywhere.

Let $\cG_{k,\ell} \colon \RR^{2n} \to \RR^s$ be the map defined as:
\begin{equation}
\label{eq:gril_vec}
    \cG_{k,\ell}(v_{f}) = \left [\Lambda_{k,\ell}^{\vp_1}(\mathbf{v}_f), \Lambda_{k,\ell}^{\vp_2}(\mathbf{v}_f), \hdots, \Lambda_{k,\ell}^{\vp_s}(\mathbf{v}_f) \right ]^{T}
\end{equation}
where $\{\vp_j\}_{j=1}^s$ are the $s$ sampled center points. We drop the $k, \ell$ and refer to $\cG_{k,\ell}$ as $\cG$ whenever $k, \ell$ are well understood. We note that $\cG$ is also a function of the center points $\vp_j$ for all $j$. We show that $\cG$ is piecewise affine.

For notational convenience, in what follows
we denote $f_x(\sigma)$ as $\sigma^x$ and $f_y(\sigma)$ as $\sigma^y$, and we call them the {\em simplex coordinates} of~$\sigma$. Similarly, we denote the $x$-coordinate and $y$-coordinate of $\vp$ as $\vp^x$ and $\vp^y$ respectively.

We observe that, if there are two simplices $\sigma_i, \sigma_j$ such that for some $\rho \in \mathbb{Z} \text{ and }0 \leq \rho \leq \ell$ and $a,b\in \{x,y\}$, one has
$|\sigma_i^a - \vp_t^a| = \rho \cdot |\sigma_j^b - \vp_t^b|$ 
then, say for $a=x$ and $b=y$, the point representing the vector $\mathbf{v}_f$ lies on a hyperplane in $\RR^{2n}$:
\begin{eqnarray*}
  \left\{ \mathbf{v} \in \RR^{2n}\colon \left |\mathbf{v}[2i-1] - \vp_t^x\right | = \rho \cdot \left |\mathbf{v}[2j] - \vp_t^y\right | \right \}.
\end{eqnarray*}
Corresponding to each such pair of simplex coordinates, we get one hyperplane. Here, in the example, we chose one simplex coordinate to be $x$-coordinate and the other to be $y$-coordinate. This, however, need not always be the case: we can have all four possible combinations of $x$- and $y$-coordinates, corresponding to each of which we get a hyperplane. The exact formula of all such hyperplanes is given in Appendix \ref{app:diff_gril}. Combining all these hyperplanes, we get an \emph{arrangement $\cH$ of hyperplanes} in $\RR^{2n}$~\cite{CGHandbook}. The arrangement $\cH$ partitions $\mathbb{R}^{2n}$ into \emph{relatively open} 
$r$-cells, $r \in \{0,\hdots, 2n\}$. We observe that this arrangement induces an affine stratification $\mathcal{S}_\cH$ (refer \cite{diff_calc_barcodes} for the formal definition) of $\RR^{2n}$ where the $r$-dimensional strata are precisely the $r$-cells. 

\begin{figure}
    \centering
    \includegraphics[scale=0.4]{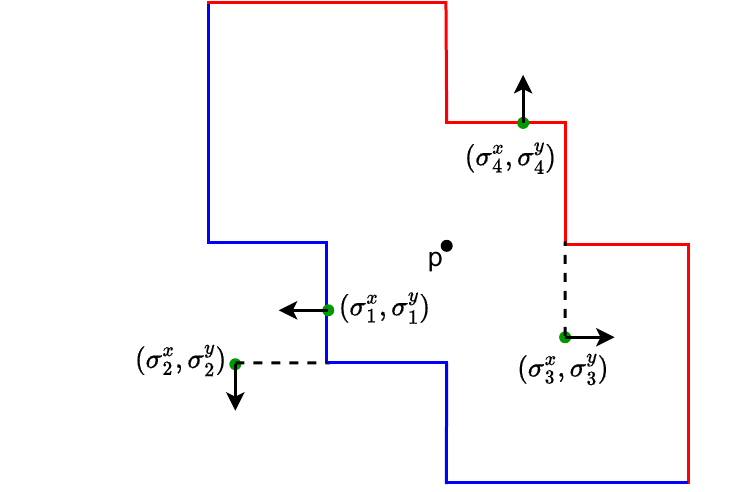}
    \caption{A 2-worm with lower boundary colored in blue and upper boundary colored in red. The figure also shows the possible cases of \textbf{constraining simplex coordinates}; $\sigma_1$ is a case of \emph{lower $x$-constraining} simplex coordinate because the $x$-coordinate $\sigma_1^x$ constrains the lower boundary of the worm and prevents the worm from expanding further to the left; $\sigma_2$ is an example of \emph{lower $y$-constraining} simplex coordinate because $\sigma_2^y$ also constrains the lower boundary of the worm and prevents the worm from expanding downwards. Similarly, $\sigma_3$ and $\sigma_4$ are \emph{upper $x$-constraining} and \emph{upper $y$-constraining} respectively. The arrows depict the gradient directions as described in Theorem \ref{thm:gril_diff}.}
    \label{fig:constraining_coord}
\end{figure}


In the following we prove that $\cG$ is a piecewise affine map relative to this arrangement, meaning that it is affine on each stratum of~$\mathcal{S}_\cH$. To this end, we introduce the notions of \emph{upper boundary}, \emph{lower boundary}, and \emph{constraining simplex coordinate} for an $\ell$-worm which will allow us to characterize the strata of $\mathcal{S}_\cH$ in terms of conditions on the bifiltration function. We give an intuitive explanation of these concepts using Figure \ref{fig:constraining_coord}. Formal definitions are given in Appendix \ref{app:diff_gril}. 

In Figure \ref{fig:constraining_coord}, a $2$-worm is shown with its lower boundary colored in blue and its upper boundary in red. As depicted, the \emph{lower boundary} consists of the part of the boundary that is in the lower half of the worm, 
while the \emph{upper boundary} consists of the rest of the boundary, i.e., a point belongs to the lower boundary (resp. upper boundary) if the open lower-set (resp. upper-set) of the point does not intersect the worm. 

For an intuitive understanding of \emph{constraining simplex coordinate}, consider the \gril{} value $d$ of an $\ell$-worm centered at $\vp$ for a given rank $k$. It follows from Definition~\ref{def:gril} that there exists at least one simplex $\sigma$ such that $\sigma^x$ or $\sigma^y$ prevents the worm from expanding any further to have a \gril{} value more than $d$ at $\vp$. The coordinate in question (either $\sigma^x$ or $\sigma^y$) is called a \emph{constraining simplex coordinate} for the $\ell$-worm centered at $\vp$. 
Here, by preventing the worm from expanding, we mean that if the worm were to expand then the value of the generalized rank over the worm would drop below the given value of rank $k$. Note that an $\ell$-worm can have multiple constraining simplex coordinates, including some coming from the same simplex~$\sigma$. 

We can characterize the top-dimensional ($2n$-dimensional) strata of $\mathcal{S}_\cH$ in terms of conditions on constraining simplex coordinates for $\ell$-worms, and consequently, in terms of conditions on the bifiltration functions.

\begin{proposition}\label{prop:unique_constraing_splx}
    The top-dimensional strata of $\mathcal{S}_\cH$ consist precisely of those bifiltration functions that have a unique constraining simplex coordinate for each $\ell$-worm.   
\end{proposition}

We state the main theorem of this subsection.
\begin{theorem}\label{thm:gril_affine}
    Let $K$ be a simplicial complex with $n$ simplices. Let $k, \ell \in \mathbb{N}$, and let $\{\mathbf{p}_j\}_{j=1}^s$ be the $s$ sampled center points for the $\ell$-worms. 
    Then, $\cG$, as defined in Eq.\eqref{eq:gril_vec}, is a piecewise affine map relative to the arrangement~$\cH$.
\end{theorem}

\textit{Overview of the proof:} $\cG$ depends affinely on the simplex coordinates in each top-dimensional stratum, because there is a unique constraining simplex coordinate for each $\ell$-worm. By continuity of $\cG$ (Proposition \ref{prop:gril_cont}), the restriction of $\cG$ to each (affine) lower dimensional stratum is affine. 

\subsection{Stochastic sub-gradient descent}\label{subsec:sgd}
In our machine learning pipeline, the \gril{} map $\cG$ is post-composed with a loss function $N\colon \RR^s\to\RR$, derived e.g. from some neural network. In the corollary below, we  give sufficient conditions on $N$ that ensure the convergence of stochastic sub-gradient descent on $N\circ \cG$. 

\begin{corollary}\label{cor:sgd_converge}
    If $N \colon \RR^s \to \RR$  is definable and locally Lipschitz continuous, then, under the assumptions of Proposition~\ref{prop:sgd_convergence} and Theorem \ref{thm:gril_affine}, stochastic sub-gradient descent on $N \circ \cG$ converges almost surely to critical points of $N \circ \cG$. 
\end{corollary}
\textit{Overview of the proof:} In the discussion towards the end of section \ref{subsec:stoc_subgrad}, we saw that any piecewise affine map is definable. 
Since $\cG$ is piecewise affine, it is locally Lipschitz and definable. We know that the composition of two definable functions is definable, and that the composition of two locally Lipschitz functions is locally Lipschitz. These facts, put together with Proposition \ref{prop:sgd_convergence}, immediately lead to the fact that stochastic sub-gradient descent converges almost surely to critical points of $N \circ \cG$. 

\begin{figure}
    \centering
    \includegraphics[scale=0.7]{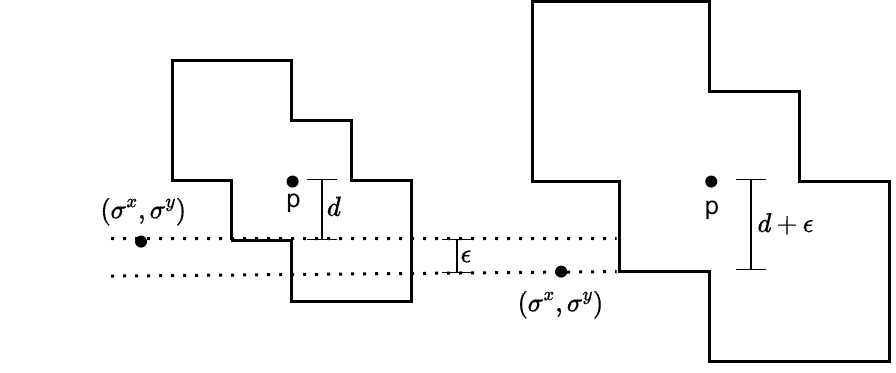}
    \caption{An intuitive understanding of the gradient assignment described in Theorem \ref{thm:gril_diff}. The $y$-coordinate $\sigma^y$ is at a distance of $d$ from the $y$-coordinate of $\vp$ in the left figure; $\sigma$ is the only constraining simplex for the $2$-worm. In the figure on the right, $\sigma^y$ has \emph{reduced} by $\epsilon$ and is now at a distance $d+\epsilon$ from $\vp^y$. As a consequence, the value of \gril{} increases from $d$ to $d + \epsilon$. Thus, $\frac{\partial \Lambda^{\vp}_{k,\ell}(\mathbf{v}_f)}{\partial \sigma^y} = -1$.}
    \label{fig:grad_intuitive_expln}
\end{figure}

\subsection{Differential of \gril{}}\label{subsec:grad_gril}
 In order to back-propagate gradients through the \gril{} computation, we need an explicit formula for the differential of $\cG$. Such a formula is needed only in the top-dimensional strata, where $\cG$ is actually smooth. In lower-dimensional strata, sub-gradients can be approximated using, e.g., gradient sampling~\cite{burke2020gradient}. 

In section \ref{subsec:gril_piecewise_affine}, we saw that $\cG$ is piecewise affine, which implies that its differential is constant in each top-dimensional stratum. In order to compute it, we need to determine the sign of the gradient at different simplex coordinates. To do so, given a worm, we distinguish between simplex coordinates that constrain the lower boundary of the worm and those that constrain its upper boundary. Since each simplex has two coordinates, $x$ and $y$, and these can be constraining either the upper boundary or the lower boundary, we get four different cases: \emph{lower $x$-constraining, lower $y$-constraining, upper $x$-constraining} and \emph{upper $y$-constraining.} All these cases are shown in Figure \ref{fig:constraining_coord}.

We deduce an explicit formula for the differential of $\mathcal{G}$ in the top-dimensional strata.

\begin{theorem}\label{thm:gril_diff}
    Let $K$ be a simplicial complex with $n$ simplices. Let $k, \ell \in \mathbb{N}$ and $\{\mathbf{p}_j\}_{j=1}^s$ be the $s$ sampled center points for the $\ell$-worms. Then, the differential of $\cG$ at any $\mathbf{v}_f$ in a top-dimensional stratum is given by: 
    \begin{equation*}
        \begin{pmatrix}
            \frac{\partial \Lambda^{\vp_1}_{k,\ell}(\mathbf{v}_f)}{\partial \sigma_{1}^x} & \frac{\partial \Lambda^{\vp_1}_{k,\ell}(\mathbf{v}_f)}{\partial \sigma_{1}^y} & \frac{\partial \Lambda^{\vp_1}_{k,\ell}(\mathbf{v}_f)}{\partial \sigma_{2}^x} & \hdots & \frac{\partial \Lambda^{\vp_1}_{k,\ell}(\mathbf{v}_f)}{\partial \sigma_{n}^y} \\
            \vdots & & & &\vdots \\
            \frac{\partial \Lambda^{\vp_s}_{k,\ell}(\mathbf{v}_f)}{\partial \sigma_{1}^x} & \hdots & \hdots&  & \frac{\partial \Lambda^{\vp_s}_{k,\ell}(\mathbf{v}_f)}{\partial \sigma_{n}^y}
        \end{pmatrix}_{s \times 2n}
    \end{equation*}
     where, 
    \[ 
    \frac{\partial \Lambda^{\vp_j}_{k,\ell}(\mathbf{v}_f)}{\partial \sigma_{i}^x} = 
    \begin{cases}
        -1, & \text{if } \sigma_{i} \text{ is lower x-constraining for } \boxed{\vp_j}^\ell_{d_j},\\
        +1, & \text{if } \sigma_{i} \text{ is upper x-constraining for } \boxed{\vp_j}^\ell_{d_j},\\
        0, & \text{otherwise},
    \end{cases}
    \]
    \[ 
    \frac{\partial \Lambda^{\vp_j}_{k,\ell}(\mathbf{v}_f)}{\partial \sigma_{i}^y} = 
    \begin{cases}
        -1, & \text{if } \sigma_{i} \text{ is lower y-constraining for } \boxed{\vp_j}^\ell_{d_j}\\
        +1, & \text{if } \sigma_{i} \text{ is upper y-constraining for } \boxed{\vp_j}^\ell_{d_j}\\
        0, & \text{otherwise},
    \end{cases}
    \]
    and $d_j$ is the \gril{} value at $\vp_j$.
    
\end{theorem}

Recall that each $\mathbf{v}_f$ in a top-dimensional stratum of $\mathcal{S}_\cH$ corresponds to a bifiltration function $f$ with a unique constraining simplex. Refer to Figure \ref{fig:grad_intuitive_expln} for an intuitive explanation about the sign of the gradient and the different cases in Theorem \ref{thm:gril_diff}, as shown by the arrows in Figure \ref{fig:constraining_coord}.

\begin{corollary}
    Given the conditions of Theorem~\ref{thm:gril_diff}, partial derivatives of $\cG$ with respect to $\vp_j^x$ and $\vp_j^y$ also exist and are given by:
    \[
    \frac{\partial \Lambda^{\vp_j}_{k,\ell}(\mathbf{v}_f)}{\partial \vp_{j}^x} = 
    \begin{cases}
        +1, & \text{if there exists a lower x-constraining} \\
         & \text{simplex for } \boxed{\vp_j}^\ell_{d_j}, \\
        -1, & \text{if there exists an upper x-constraining} \\ 
        & \text{simplex for }\boxed{\vp_j}^\ell_{d_j}, \\
        0, & \text{otherwise,}
    \end{cases}
    \]
    \[ 
    \frac{\partial \Lambda^{\vp_j}_{k,\ell}(\mathbf{v}_f)}{\partial \vp_{j}^y} = 
    \begin{cases}
        +1, & \text{if there exists a lower y-constraining} \\
         & \text{simplex for } \boxed{\vp_j}^\ell_{d_j}, \\
        -1, & \text{if there exists an upper y-constraining} \\ 
        & \text{simplex for }\boxed{\vp_j}^\ell_{d_j}, \\
        0, & \text{otherwise.}
    \end{cases}
    \]
\end{corollary}

\subsection{Practical Considerations}\label{subsec:pract_consider}
In section \ref{subsec:gril_piecewise_affine} we saw that $\cG$ is piecewise affine and in section \ref{subsec:sgd} we deduced that stochastic sub-gradient descent converges almost surely to critical points of $N\circ \cG$ for any definable and locally Lipschitz loss function $N$. We gave an explicit formula for the differential of \gril{} in the top-dimensional strata in section \ref{subsec:grad_gril}, and at the beginning of that section we argued that gradient sampling can be used to approximate sub-gradients in lower-dimensional strata. In practice, for the sake of computational efficiency, we use a simple variant of gradient sampling, which consists of sampling a single nearby point and taking its gradient.

So far, we considered a bifiltration function on the entire simplicial complex $K$. However, in practice, we may need to extend the filtration function to the entire simplicial complex based on filtration function values on certain simplices. For example, in a lower-star filtration, the filtration function values on the higher dimensional simplices are dictated by the filtration function values on the vertices of the simplicial complex. In a Rips filtration, the filtration function values on the vertices are 0, while the filtration function values on higher dimensional simplices are dictated by the filtration function values on the edges.

In such scenarios, the bifiltration function $f$ is given on a subcomplex $L$ of the simplicial complex $K$. This is extended to a bifiltration function $\Bar{f}$ on the entire simplicial complex in a piecewise constant manner. We use the example of lower-star bifiltration to show how our framework fits to this type of scenario. Let $f \colon K^0 \to \RR^2$ be a function on the vertices of the simplicial complex $K$ with $m$ vertices and $n$ simplices. The map $f$ is extended to a piecewise constant map $\Bar{f} \colon K \to \RR^2$ as follows: for an edge $e=(u,v)$, we let $\Bar{f}_x(e) = \max\{f_x(u), f_x(v)\}$ and $\Bar{f}_y(e) = \max\{f_y(u), f_y(v)\}$; the values on higher dimensional simplices are defined inductively. Notice that for each simplex $\sigma \in K$, there is a \emph{maximal $x$-vertex} and a \emph{maximal $y$-vertex}, the vertices that give the $x$ and $y$-values to the simplex respectively. We know that $f$ can be represented as a vector $\mathbf{v}_f \in \RR^{2m}$ and $\Bar{f}$ as a vector $\mathbf{v}_{\Bar{f}} \in \RR^{2n}$ after ordering the simplices of $K$. Consider the map $q \colon \RR^{2m} \to \RR^{2n}$ given by $q(\mathbf{v}_f) = \mathbf{v}_{\Bar{f}}$. The map $q$ is piecewise affine and thus, $\cG \circ q$ is piecewise affine as well. Therefore, we can apply stochastic sub-gradient descent on $N \circ \cG \circ q$ with analogous convergence guarantees. 

Observe that we get an arrangement of hyperplanes on $\RR^{2m}$, $\cH_q$, analogous to $\cH$ on $\RR^{2n}$ that we saw in the previous subsections. On the top-dimensional cells of $\cH_q$, the differential of $q$ is well-defined and constant and expressions analogous to the ones in Theorem \ref{thm:gril_diff} can be derived,
and the chain rule can be used for back-propagating gradients. As mentioned previously, we use a simple form of gradient sampling to approximate sub-gradients when the gradient is not well-defined. For the experimental section that follows, $N$ is a neural network with one of the standard choices for loss functions, which is definable as discussed towards the end of section \ref{subsec:bg_ominimal}.

We observe that if the bifiltration $f$ is extended in the manner described above to $\Bar{f}$, then we end up in a scenario where multiple simplices have the same $x$ or $y$-values. This is a degenerate scenario with respect to the assumptions in Theorem \ref{thm:gril_affine}. To avoid this, in practice, we add an infinitesimal perturbation to the higher dimensional simplices to ensure the values are not exactly the same. We note that this is, in spirit, similar to gradient sampling. 


\section{Experiments}
\label{sec:experiments}

In this section, we present the experimental results on various bio-activity prediction datasets and benchmark graph datasets. We begin the section by describing our experimental setup and then move on to a brief description of bio-activity prediction datasets followed by experimental results on those datasets. Towards the end of the section, we show that \grild{} can be, more generally, applied in the context of bifiltration learning on standard benchmark graph datasets. We compare with existing multiparameter persistent homology methods on these datasets. All the reported accuracy/ROC-AUC scores, in this section, are $5$-fold cross-validated scores. 

\begin{figure*}[!htb]
    \centering
    \includegraphics[scale=0.07]{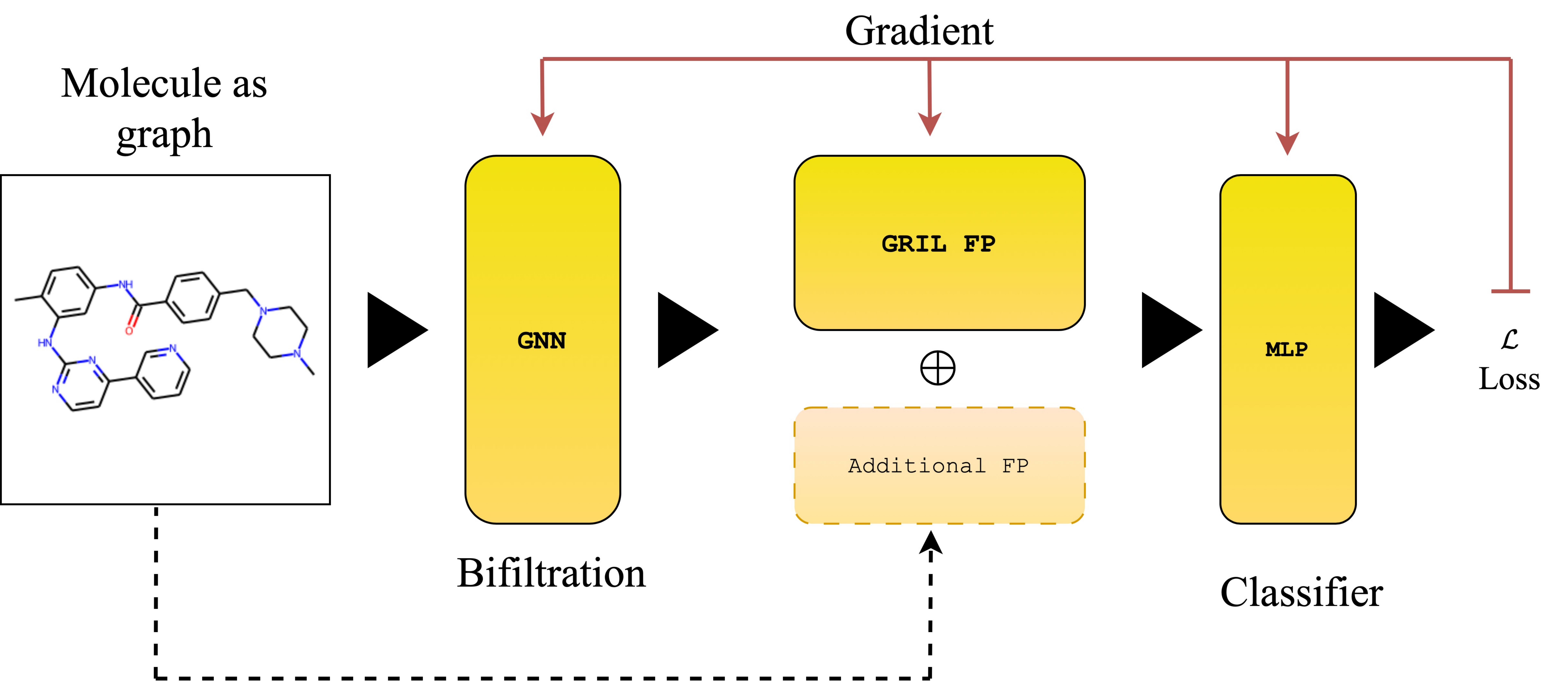}
    \caption{Architecture choice for bio-activity prediction; the bifiltration fuction $f$ is learnt compared to the standard multiparameter pipeline; $\oplus$ denotes concatenation of vectors.}
    \label{fig:framework}
\end{figure*}

\subsection{Experimental Setup}
Every instance in each dataset is an attributed graph with a label associated with it. We use a Graph Isomorphism Network (GIN)~\cite{gin_xu2018how} to obtain a bifiltration function over the vertices. We consider the lower-star bifiltration of this function. This gives us a bifiltration function on the entire graph, which we denote as $f$. Let $\mathbf{v}_f$ denote the vector representation of $f$ which we get after ordering the $n$ simplices of the graph. 
We divide the range of $f$ into a grid of size $100 \times 100$. We randomly initialize the $s$ center points from this grid and compute the vector $\cG_{k,\ell}(\mathbf{v}_f)$ (Eq.~\eqref{eq:gril_vec}) for $k=1,2,3$ and $\ell = 2$ at these center points. 
This vector $\cG_{k,\ell}(\mathbf{v}_f)$ is fed to a $3$-layer Multilayer Perceptron (MLP) classifier and the loss is back-propagated through the \gril{} computation according to the gradient described in Section \ref{sec:diff_of_gril}. Note that the positions of the $s$ center points are also optimized as the gradient can be backpropagated to the coordinates of the center points as well. Additional details can be found in Appendix \ref{app:experiments}.

\subsection{Bio-activity Prediction Datasets}
The data is extracted from the ChEMBL database~\cite{gaulton2012chembl,davies2015chembl}. Each  dataset contains the SMILES~\cite{weininger1988smiles} encoding of a novel drug (compound), and activity pairs for a target of interest. For example, the \textsc{EGFR} dataset contains all the molecules that have been tested against epidermal growth factor receptor (EGFR) kinase, and their measured bio-activity. The SMILES encodings of the molecules are then converted to graphs with \textsc{Molfeat}~\cite{molfeat}. Bio-activity is measured in \textit{half maximal inhibitory concentration $(IC_{50})$}, which qualitatively indicates how much a drug is needed, \textit{in vitro}, to inhibit a particular process by $50\%$. To facilitate the comparison of $IC_{50}$ values, it is common practice to convert $IC_{50}$ to $pIC_{50} = -\log_{10}(IC_{50})$, expressed in molar units. Threshold for activity cutoff $pIC_{50} = 6.3$ is used throughout the paper. Additional details are given in Appendix~\ref{app:experiments}.
\begin{table*}
    \centering
    \begin{tabular}{ccccc} 
    \toprule
         \textbf{Dataset} &  \textbf{GIN}  &\textbf{GIN-\gril}&  \textbf{GIN-\textsc{Mpsm}}&\textbf{\grild}   \\ 
         \midrule
         EGFR             &  $55.60 \pm 8.61$                &$58.39 \pm 2.51$&  $\mathbf{66.97 \pm 1.87}$ & $65.38 \pm 5.34$                \\ 
         ERRB2            &  $57.06 \pm 8.58$                &$61.28 \pm	3.66$&  $68.75 \pm 1.98$&$\mathbf{69.24 \pm 5.06}$                \\\ 
         CHEMBL1163125    &  $58.48\pm4.37$            &$54.13	\pm 1.09$&  $61.23 \pm 2.47$&$\mathbf{65.69 \pm 1.57}$                \\ 
         CHEMBL2148       &  $52.88\pm0.91$               &$50.24 \pm 0.18$&  $51.45 \pm 0.87$&$\mathbf{53.13 \pm 3.70}$               \\ 
 CHEMBL4005       & $55.90 \pm 4.63$               & $51.85 \pm 3.55$ &  $\mathbf{61.35 \pm 1.49}$&$59.34 \pm 5.70$               \\
 \bottomrule
    \end{tabular}
    \caption{Test ROC-AUC on ChEMBL datasets. \grild{} performs better than GIN with \emph{sum} pooling and \gril{} with bifiltration obtained from pre-trained GIN.}
    \label{tab:none_gril_comp}
\end{table*}


In the first set of experiments, we compare \grild{} with (i) a standard GNN model--GIN with sum-pooling, (ii) with GIN and \gril{} as a readout layer and with (iii) GIN and \textsc{Mpsm}~\cite{diff_signed_barcodes_24}. Note that for the results on GIN-\gril, \gril{} is used as a passive readout layer, i.e., we obtain a bifiltration function from a pre-trained GIN (pre-trained for graph classification on the same dataset) and compute \gril{} on it. We use the exact same pipeline as GIN-\gril{} and train it end-to-end using \grild{} and GIN-\textsc{Mpsm} and report the results. The classifier, $3$-layer MLP, is the same for all the experiments. Refer to Figure \ref{fig:framework} for the pipeline.
From Table~\ref{tab:none_gril_comp}, we can see that adding topological information after GIN has been trained can improve the performance in some cases and need not be beneficial in others. However, it is clear that adding topological information in an end-to-end framework appears to be beneficial for bio-activity prediction. The performance of \grild{} is comparable to \textsc{Mpsm}.

Further, we perform a series of experiments where we augment other popular molecular fingerprints such as ECFP and Morgan3~\citep{rogers2010extended,morgan1965generation} with the \grild{} framework. In these experiments, we compare the performance of the model with and without  \grild{} augmentation. We report the results in Table \ref{tab:results}. We can see from the table that augmenting with \grild{} seems to, generally, increase the performance of the model. Fingerprints such as ECFP and Morgan3 can essentially represent an infinite number of different molecular features, including stereochemical information~\citep{rogers2010extended} but they are not effective in capturing global features of molecules such as size and shape~\citep{Capecchi_Probst_Reymond_2020}. Combining \grild{} with these fingerprints augments the model with topological information, yielding an improved performance.


\begin{table*}[!htb]
\centering
\resizebox{0.8\textwidth}{!}{
\begin{tabular}{@{}c|cc|cc@{}}
\toprule
\textbf{Dataset} & \textbf{ECFP} & \textbf{ECFP+\grild} & \textbf{Morgan3} & \textbf{Morgan3+\grild} \\ \midrule
EGFR             & $83.27 \pm 1.10$              & $\mathbf{84.37 \pm 1.22}$& $82.39 \pm 1.35$                  &  $\mathbf{83.57 \pm 1.40}$\\
ERRB2            & $83.53 \pm 1.31$              &  $\mathbf{86.24 \pm 2.27}$& $83.19 \pm 1.26$                  & $\mathbf{85.51\pm 1.47}$\\
CHEMBL1163125    & $83.94 \pm 1.23$              & $\mathbf{86.33 \pm 0.67}$&   $83.55 \pm 1.20$                & $\mathbf{84.85 \pm 0.82}$\\
CHEMBL203        & $81.74 \pm 0.90$              & $\mathbf{81.86 \pm 0.90}$& $80.85 \pm 1.32$                 & $\mathbf{81.64 \pm 0.45}$\\
CHEMBL2148       & $73.96 \pm 2.59$& $\mathbf{74.13 \pm 2.33}$& $72.79 \pm 1.97$                  & $\mathbf{74.50 \pm 2.61}$\\
CHEMBL279        & $76.72 \pm 1.34$              &  $\mathbf{78.75 \pm 0.77}$&  $76.76 \pm 0.50$                &                   $\mathbf{78.05 \pm	1.18}$\\
CHEMBL2815       &$73.69 \pm 1.53$               & $\mathbf{76.51 \pm 1.48}$&  $73.42 \pm 0.56$                 &  $\mathbf{75.22 \pm 2.85}$\\
CHEMBL4005       &$\mathbf{80.45 \pm 1.30}$               &  $80.41 \pm 0.97$&  $79.95 \pm 1.30$                & $\mathbf{80.55 \pm 0.82}$\\
CHEMBL4722       & $78.05 \pm	1.64$& $\mathbf{78.51 \pm 1.83}$& $78.76 \pm	1.37$                 & $\mathbf{79.17 \pm 1.30}$\\ \bottomrule
\end{tabular}}
\vspace{-0.1cm}
\caption{Test ROC-AUC scores on ChEMBL datasets; augmenting  ECFP and Morgan3 fingerprints with \grild{} increases the classification performance for most of the datasets.}
\label{tab:results}
\vspace{-0.1cm}
\end{table*}

\begin{table*}[!htb]
\resizebox{\textwidth}{!}{
\centering
\begin{tabular}{@{}cccccccc@{}}
\toprule
\textbf{Dataset} &\textbf{\grild}  &\textbf{\textsc{Mpsm}}& \textbf{\gril} & \textbf{MP-I} & \textbf{MP-L} & \textbf{MP-K} & \textbf{P} \\ \midrule
MUTAG            &$\mathbf{85.09 \pm 5.99}$               &$78.44 \pm 3.33$& $83.49 \pm 3.64$          & $74.99 \pm 2.79$             & $82.42 \pm 3.72$         & $79.27 \pm 2.45$             & $66.50 \pm 0.87$           \\
PROTEINS         &$69.45 \pm 4.11$                     &$68.37 \pm 3.31$& $66.31 \pm 2.34$          & $70.80 \pm 3.09$             & $\mathbf{70.80 \pm 1.31}$         & $61.70 \pm 2.98$              & $59.57 \pm 0.08$            \\
DHFR             &$61.24 \pm 4.37$                    &$60.97 \pm 1.98$& $\mathbf{61.64 \pm 1.66}$          & $60.98	\pm 0.10$              & $61.11 \pm 0.25$         & $60.98 \pm 0.10$             &  $60.98	\pm 0.10$          \\
COX2             &$78.16 \pm 0.41$               &$77.73 \pm 0.36$& $78.16 \pm 0.41$          &   $78.16 \pm 0.41$           & $78.16 \pm 0.41$         & $78.16 \pm 0.41$           & $78.16 \pm 0.41$           \\ 
 IMDB-BINARY &$\mathbf{62.60 \pm 6.56}$  & $60.80 \pm 3.82$& $50.00 \pm 0.00$& $56.60 \pm 2.94$& $50.00 \pm 0.00$ & $50.30 \pm 1.12$ & $50.00 \pm 0.00$\\\bottomrule
\end{tabular}
}
\vspace{-0.2cm}
    \caption{Test accuracies of \grild{} on benchmark graph datasets.}
\label{tab:graph_benchmark_results}
\vspace{-0.3cm}
\end{table*}

\begin{table}[h!]
    \centering
    \begin{tabular}{ccc}
    \toprule
        \textbf{Dataset} & \textbf{\grild} & \textbf{\textsc{Mpsm}}\\
        \midrule
         MUTAG & 00:05:08 & 00:08:47\\
         COX2 & 00:11:34 & 00:18:32\\
         DHFR & 00:16:48 & 00:27:12\\
         PROTEINS & 00:30:03 & 00:50:26\\
         \bottomrule
    \end{tabular}
    \vspace{0.2cm}
    \caption{Reported times (hh:mm:ss) are training times per fold averaged over 5 folds that we used for training. All the experiments have been performed on a machine with AMD EPYC 7313 16-Core Processor and NVIDIA A10 GPU.}
    \label{tab:training_times}
    \vspace{-0.4cm}
\end{table}

\subsection{Benchmark Graph Datasets}
\grild{} can be used more generally for filtration learning on graph datasets. We perform a series of experiments with benchmark graph datasets such as \textsc{Mutag, Proteins, Dhfr, Cox2}~\cite{TUDatasets} and compare with existing multiparameter persistence methods. Details about the datasets are given in Appendix \ref{app:experiments}. We report the results in Table \ref{tab:graph_benchmark_results}. For a valid comparison, we use a $3$-layer MLP as the classifier for all the multiparameter signatures. Hence, the numbers reported look different from the ones in \cite{gril23}, \cite{Carriere_Multipers_Images} where the authors consider XGBoost~\citep{xgboost}. In fact, in \cite{gril23}, the authors show that XGBoost performs much better than a $3$-layer MLP classifier for \gril{} vectors on these datasets. We can see from the table that learning the bifiltration function seems to perform better than multiparameter persistence methods on popular choices of bifiltration functions on most datasets. We can also see that \grild{} performs better than \gril{}, supporting our argument for an end-to-end learning framework. We also see that \grild{} performs better than \textsc{Mpsm}. The training times for the two end-to-end learning methods, \grild{} and \textsc{Mpsm} are reported in Table~\ref{tab:training_times}. We can see that the training times for \grild{} are 
improved by about $40\%$ compared to the ones of \textsc{Mpsm}.
We refer the reader to Figure~ \ref{fig:dhfr-gril} for visualizations of bifiltration functions learnt by \grild.

\begin{figure}[!htb]
    \centering
    \includegraphics[width=\textwidth]{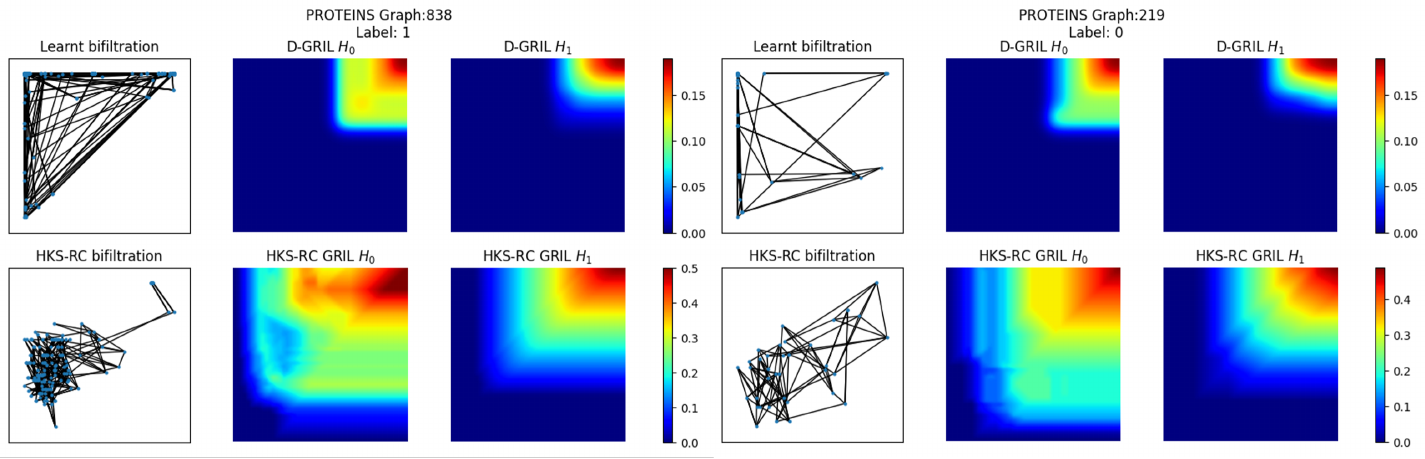}
    \caption{The figure compares the learnt bifiltration function with the Heat-Kernel Signature-Ricci Curvature (HKS-RC) bifiltration on two randomly selected graph instances (838 and 219) of \textsc{Proteins} dataset. These two instances have different labels, 1 and 0 respectively. In the first column, bifiltration function on the vertices of these graphs are plotted. We can see that the learnt bifiltration function is very different from the HKS-RC bifiltration. In the second and third column, \gril{} vectors are shown using a heatmap for $H_0$ and $H_1$ respectively. We can observe that these signatures are very different in nature. This provides some evidence that the model is learning a totally different bifiltration function as compared to HKS-RC, which is one of the common choices for bifiltration function on graphs.}
    \label{fig:dhfr-gril}
\end{figure}

\section{Conclusion}
\label{sec:conclusion}
In this paper, we propose a differentiable framework using \gril. We show that stochastic sub-gradient descent converges almost surely on \gril{}, and we compute an explicit formula for the differential of \gril. We use this formula to back-propagate gradients through the \gril{} computation, yielding the differentiable framework \grild. We use \grild{} to show that adding topological information in an end-to-end manner is beneficial for bio-activity prediction. Further, we show that \grild{} can be used in a more general setting, for filtration learning on graphs. Moreover, we show that \grild{} is faster and lighter, both in theory and in practice, than \textsc{Mpsm}. Our results indicate that learnt bifiltration functions on graphs, generally, give a better performance than the standard choices for bifiltration functions. This indicates, and we believe, that learning filtration functions can prove to be more informative for topological methods in machine learning. We hope that this work sparks interest and motivates research in this direction.

\section{Acknowledgement}
This work is supported partially by NSF grants CCF 2049010 and DMS 2301360.



\newpage
\bibliography{ref}
\bibliographystyle{icml2024}

\newpage

\appendix
\onecolumn

\section{Formal Definitions and Proofs}
\label{app:diff_gril}

\begin{definition}[Interval in $\RR^2$]
        An interval in ${\RR^2}$ is a subset $\emptyset \neq I \subseteq {\RR^2}$ that satisfies the following:
    \begin{enumerate}
        \item If $\mathbf{u},\mathbf{v} \in I$ and $\mathbf{u} \leq \mathbf{w} \leq \mathbf{v}$, then $\mathbf{w} \in I$;
        \item If $\mathbf{u},\mathbf{v} \in I$, then there exists a finite sequence ($\mathbf{u}=\mathbf{u}_0,\mathbf{u}_1, , ... , \mathbf{u}_{m}=\mathbf{v}) \in I$ so that every consecutive points $\mathbf{u}_i,\mathbf{u}_{i+1}$
        are comparable in the partial order for $i\in \{0,\ldots,m-1\}$.
    \end{enumerate}
    \label{def:interval}
    \end{definition}
    
\begin{definition}[Upper-set and lower-set]
    Given a poset $(P, \leq)$, the \emph{upper-set} of $x\in P$ is defined as
    \begin{equation*}
        x ^{\uparrow P} \coloneqq \{y \in P \colon x \leq y\}.
    \end{equation*}
    Similarly, the \emph{lower-set} of $x \in P$ is defined as
    \begin{equation*}
        x^{\downarrow P} \coloneqq \{y \in P \colon y \leq x \}.
    \end{equation*}
\end{definition}

\begin{definition}[Upper and lower boundary of $\ell$-worm]
    Let an $\ell$-worm centered at $\vp$ with width $d$, denoted as $\boxed{\vp}_{d}^{\ell}$, be given. A point $\mathbf{t}$ is said to be on the \emph{upper boundary} of the worm if $\mathring{\mathbf{t}}^{(\uparrow \RR^2)} \cap \boxed{\vp}_d^\ell = \emptyset$ where $\mathring{\mathbf{t}}^{(\uparrow \RR^2)}$ denotes the open upper-set of $\mathbf{t}$ in $\RR^2$. The collection of all such points constitutes the \emph{upper boundary} of the worm. Similarly, a point $t$ is on the \emph{lower boundary} if $\mathring{\mathbf{t}}^{(\downarrow \RR^2)} \cap \boxed{\vp}_d^\ell = \emptyset$ and the collection of all such points constitutes the \emph{lower boundary} of the worm.
\end{definition}

\begin{definition}[Constraining Simplex Coordinate]
    Given a bifiltration function $f$, $k \in \mathbb{N}$, $\ell \in \mathbb{N}, \vp \in \RR^2$ let $\lambda^{M_f}(\vp, k, \ell) = d$. Let $\sigma$ be a simplex with $\boxed{\vp}^{\ell}_{d} \cap f(\sigma)^{\uparrow \RR^2} \neq \emptyset$ such that one of the following two conditions holds:
    \begin{enumerate}
        \item $f_x(\sigma) = \vp_x \pm j \cdot d$
        \item $f_y(\sigma) = \vp_y \pm j \cdot d$
    \end{enumerate}
     for some $j \in \{1, \hdots, \ell - 1\}$. Then $\sigma$ is called a \emph{constraining simplex} for $\boxed{\vp}^{\ell}_d$. If $\sigma$ satisfies (1), $\sigma^x (=f_x(\sigma))$ is the \emph{constraining simplex coordinate} or equivalently, $\sigma$ is called \emph{x-constraining} and if it satisfies (2), $\sigma^y (=f_y(\sigma))$ is the \emph{constraining simplex coordinate} or equivalently $\sigma$ is called \emph{y-constraining}.
\end{definition}

\begin{definition}
    Let $f$ be a bifiltration function. Let $\sigma$ be a constraining simplex for $\boxed{\vp}^\ell_d$. $\sigma$ is said to be an \emph{upper constraining} simplex if $f(\sigma)^{\uparrow \RR^2}$ intersects only the upper boundary of $\boxed{\vp}^\ell_d$. $\sigma$ is called a \emph{lower constraining simplex} if $f(\sigma)^{\uparrow \RR^2}$ intersects both lower and upper boundary of $\boxed{\vp}^\ell_d$. $\sigma$ is said to be \emph{lower $x$-constraining} if $\sigma$ is lower constraining and $\sigma^x$ is the constraining simplex coordinate. The notions of \emph{upper $x$-constraining, lower $y$-constraining} and \emph{upper $y$-constraining} are similarly defined.
\end{definition}

\textbf{Formulae for the arrangement of hyperplanes.} Here, we provide the formulae for the different hyperplanes that form the arrangement of hyperplanes described in section \ref{sec:diff_of_gril}. Recall that we have $s$ sampled center points and a bifiltration function on $n$ simplices, $\sigma_1, \hdots, \sigma_n$. 

\begin{enumerate}
    \item 
    \begin{equation*}
    \left \{ \mathbf{v} \in \RR^{2n} \colon |\mathbf{v}[i] - \vp_{j}^x| = m \cdot |\mathbf{v}[k] - \vp_{j}^x|
    \begin{aligned}
    & \text{ ~ for some } j \in \{1,\hdots, s\}, \\
    & \text{ ~ } i,k \equiv 0 (\textrm{mod } 2),\\
    & \text{ ~ }m \in \{0, \hdots, \ell\}
    \end{aligned}\right \}
    \end{equation*}
    
    \item 
    \begin{equation*}
    \left \{ \mathbf{v} \in \RR^{2n} \colon |\mathbf{v}[i] - \vp_{j}^y| = m \cdot |\mathbf{v}[k] - \vp_{j}^y|
    \begin{aligned}
    & \text{ ~ for some } j \in \{1,\hdots, s\}, \\
    & \text{ ~ }i,k \equiv 1 (\textrm{mod } 2),\\
    & \text{ ~ }m \in \{0, \hdots, \ell\}
    \end{aligned}\right \}
    \end{equation*}

    \item 
    \begin{equation*}
    \left \{ \mathbf{v} \in \RR^{2n} \colon |\mathbf{v}[i] - \vp_{j}^x| = m \cdot |\mathbf{v}[k] - \vp_{j}^y|
    \begin{aligned}
    & \text{ ~ for some } j \in \{1,\hdots, s\}, \\
    & \text{ ~ }i\equiv 0 (\textrm{mod } 2),\\
    & \text{ ~ }k\equiv 1 (\textrm{mod } 2),\\
    & \text{ ~ }m \in \{1, \hdots, \ell\}
    \end{aligned}\right \}
    \end{equation*}
    
    \item 
    \begin{equation*}
    \left \{ \mathbf{v} \in \RR^{2n} \colon |\mathbf{v}[i] - \vp_{j}^y| = m \cdot |\mathbf{v}_k - \vp_{j}^x|
    \begin{aligned}
    & \text{ ~ for some } j \in \{1,\hdots, s\}, \\
    & \text{ ~ }i\equiv 1 (\textrm{mod } 2),\\
    & \text{ ~ }k\equiv 0 (\textrm{mod } 2),\\
    & \text{ ~ }m \in \{1, \hdots, \ell\}
    \end{aligned}\right \}
    \end{equation*}
\end{enumerate}
The first two sets correspond to the conditions where two simplices are $x$-constraining and $y$-constraining respectively. The last two sets correspond to the condition where one simplex is $x$-constraining and one simplex is $y$-constraining.

\begin{remark}
    The stratification $\mathcal{S}_{\cH}$ associated with the arrangement of hyperplanes is affine. In each stratum of this stratification, the relative ordering of the simplices along each coordinate is fixed. This is because, the equations of the hyperplanes ensure that no two simplices have equal coordinate values as that would mean that the distance of one of the simplices to any center point along that coordinate would be equal to (hence an integral multiple of) the distance of the other simplex. Hence, the relative ordering is fixed in each stratum. This property is important to prove that $\cG$ is affine in each stratum similar to the case in $1$-parameter persistent homology setting.
\end{remark}

\begin{proof}[Proof of Theorem \ref{thm:gril_affine}]\label{proof:gril_affine}
    It is sufficient to prove that $\cG$ is piecewise affine on the top-dimensional strata of $\mathcal{S}_{\cH}$. This is because, $\cG$ being affine on each top-dimensional stratum put together with the facts that $\mathcal{S}_{\cH}$ is affine and $\cG$ is continuous, immediately gives us that $\cG$ is also affine on lower dimensional strata of $\mathcal{S}_\cH$. We prove that each coordinate function of $\cG$ is piecewise affine. For that purpose, let us consider only one $\ell$-worm centered at $\vp$. Let $\mathcal{F}$ be a top-dimensional stratum of $\mathcal{S}_\cH$. Let $\mathbf{v}_f, \mathbf{v}_{f'} \in \mathcal{F}$ be the vector representations of two bifiltration functions. Note that they have the same unique constraining simplex coordinate, $r$, for $\boxed{\vp}^\ell_d$ and $\boxed{\vp}^\ell_{d'}$ respectively. Without loss of generality, assume that $r$ is a lower $x$- constrained coordinate. Clearly, $\cG(\mathbf{v}_f) = d$ and $\cG(\mathbf{v}_{f'})=d'$. We observe that in any stratum of $\mathcal{S}_\cH$, the relative ordering of the simplices is fixed. This is because, the equations of the hyperplanes ensure that no two simplex coordinates are equal inside a stratum. As a consequence, the relative ordering among the simplices gets fixed along each coordinate in each stratum. For the relative order to change, one has to cross one of the neighboring hyperplanes. Let $\mathbf{v}_f + \mathbf{v}_{f'} = \mathbf{v}_{f''}$. Since the relative ordering is among the simplices is fixed, adding two vectors with the same relative ordering will not change the ordering in the resultant vector. Thus, the constraining simplex coordinate for $\mathbf{v}_{f''}$ is also $r$. Thus, the value of $\cG$ at $\mathbf{v}_{f''}$ would be determined by the value of the $r$th coordinate of $\mathbf{v}_{f''}$. Now, $\vp^x - \mathbf{v}_f[r] = d$ and $\vp^x - \mathbf{v}_{f'}[r] = d'$. Thus,
    \begin{align*}
        \cG(\mathbf{v}_f) + \cG(\mathbf{v}_{f'}) &= d + d' \\
        &= 2\vp^x - \mathbf{v}_f[r] - \mathbf{v}_{f'}[r] \\
         &= \vp^x + \vp^x - \mathbf{v}_{f''}[r] \\
         &= \vp^x + \cG(\mathbf{v}_{f''}).
    \end{align*}
    We note that $\vp^x$ is a constant because the sampled center points are fixed. Hence, each coordinate function of $\cG$ is affine on each top-dimensional stratum of $\mathcal{S}_\cH$. Thus, $\cG$ is a piecewise affine map.
\end{proof}

\begin{proof}[Proof of Theorem \ref{thm:gril_diff}]\label{proof:gril_diff}
    For a given rank $k$, let us consider the worm $\boxed{\vp_j}_{d_j}^\ell$ with $\sigma_{i}$ as the constraining simplex. WLOG assume that $\sigma_{i}$ is lower $x$-constraining. Then, we have $d_j = \vp_{j}^x - \sigma_{i}^x$ and $\RKG^{M_f}\left(\boxed{\vp_j}_{d}^\ell\right) < k$ for $d > d_j$ and $\RKG^{M_f}\left(\boxed{\vp_j}_{d}^\ell\right) \geq k$ for $d \leq d_j$. Now, consider the interval $\mathcal{I}_j = (d_j - \epsilon, d_j+\epsilon)$ where $\epsilon = \min(\min \limits_{t \neq i}(|\sigma_{i}^x - \sigma_{t}^x|), \vp_j^x- \sigma_{i}^x)$. Consider another bifiltration function $f'$ such that $f'(\sigma_t) = f(\sigma_t)$ for all $t \neq i$, $f'_x(\sigma_{i}) \in \mathcal{I}_j$ and $f'_y(\sigma_{i}) = f_y(\sigma_{i})$. Let $\mathbf{v}_f$ and $\mathbf{v}_{f'}$ denote the vector representations of bifiltration functions $f$ and $f'$ respectively.  Let $d'_j$ denote the value of \gril{} corresponding to the worms at $\vp_j$ for the bifiltration function $f'$. Then, we can see that $\RKG^{M_{f'}}\left( \boxed{\vp}_{d'}^\ell \right) < k$ for $d' > d_j'$ and $\RKG^{M_{f'}}\left( \boxed{\vp}_{d'}^\ell \right) \geq k$ for $d' \leq d_j'$. This is because, $\sigma_{i}$ is moved in a small interval such that its relative order with respect to other simplices or the center point $\vp_j$ does not change. Now, let $d_j = d'_j - \eta$ where $\eta < \epsilon$. Then, by definition of $d'_j$ and $d_j$, we have $\vp_j^x - \sigma_{i}^{x} = \vp_j^x - (\sigma_{i}^{'x} + \eta)$. Thus, we have $\Lambda^{\vp_j}_{k,\ell}(\mathbf{v}_{f'}) - \Lambda^{\vp_j}_{k,\ell}(\mathbf{v}_f) = -(\sigma_{i}^{'x} - \sigma_{i}^{x})$ which gives us the formula $\frac{\partial \Lambda^{\vp_j}_{k,\ell}(\mathbf{v}_f)}{\partial \sigma_{i}^x} = -1$ if $\sigma_{i}$ is lower x-constraining. One can similarly argue about upper x-constraining and about y-constraining simplices. For $\boxed{\vp_j}_{d_j}^\ell$, only $\sigma_{i_j}$ is participating and no other simplex is and thus, the derivative $\frac{\partial \Lambda^{\vp_j}_{k,\ell}(\mathbf{v}_f)}{\partial \sigma_{t}^x} = 0$ for $t \neq i$.
\end{proof}


\section{More on Experiments}
\label{app:experiments}
\textbf{Benchmark graph datasets:} In Table \ref{tab:app:graph_data}, we provide information about benchmark graph datasets.

\begin{table}[!htb]
\centering
\begin{tabular}{@{}ccccc@{}}
\toprule
\textbf{Dataset}  & \textbf{Num Graphs} & \textbf{Num Classes} & \textbf{Avg. No. Nodes} & \textbf{Avg. No. Edges} \\ \midrule
\textsc {Proteins}& $1113$& $2$& $39.06$& $72.82$\\
\textsc{ Cox2}& $467$& $2$& $41.22$& $43.45$\\
\textsc{ Dhfr}& $756$& $2$& $42.43$& $44.54$\\
\textsc{ Mutag}& $188$& $2$& $17.93$& $19.79$\\ 
 \textsc{Imdb-Binary}& $1000$& $2$& $19.77$&$96.53$\\ \bottomrule
\end{tabular}
\vspace{0.2cm}
\caption{Description of Benchmark Graph Datasets}
    \label{tab:app:graph_data}
\end{table}



\textbf{Bio-activity Prediction Datasets:} The datasets are publicly available in ChEMBL website and can be downloaded following the tutorial mentioned in~\cite{TeachOpenCADD2022}. Details about the datasets are given in Table \ref{tab:dataset-details}. We convert the molecules to graphs with \textsc{Molfeat}. The initial $82$ dimensional node features that is fed to the GNN to get the input bifiltration function are (i) atom-one-hot, (ii) atom-degree-one-hot, (iii) atom-implicit-valence-one-hot, (iv) atom-hybridization-one-hot, (v) atom-is-aromatic, (vi) atom-formal-charge, (vii) atom-num-radical-electrons, (viii) atom-is-in-ring, (ix) atom-total-num-H-one-hot, (x) atom-chiral-tag-one-hot and (xi) atom-is-chiral-center. This is the default setting of \textsc{Molfeat} and we do not claim, in any way, that these are the optimal node features that are to be used.

\begin{table}[!htb]
\centering
\begin{tabular}{@{}cccccc@{}}
\toprule
\textbf{Dataset}       & \textbf{Num Graphs} & \textbf{Active} & \textbf{Inactive} & \textbf{Avg. Num Nodes} & \textbf{Avg. Num Edges} \\ \midrule
EGFR          & 4635       & 2631   & 2004     & 28.97          & 31.73          \\
ERRB2         & 1818       & 1140   & 678      & 33.33          & 36.70          \\
CHEMBL1163125 & 2719       & 1507   & 1212     & 30.77          & 34.23          \\
CHEMBL203     & 6816       & 4234   & 2582     & 31.88          & 34.98          \\
CHEMBL2148    & 3200       & 2380   & 820      & 29.89          & 33.28          \\
CHEMBL279     & 7461       & 5104   & 2357     & 31.82          & 35.11          \\
CHEMBL2815    & 3143       & 2484   & 659      & 33.72          & 37.30          \\
CHEMBL4005    & 4790       & 3195   & 1595     & 31.79          & 35.26          \\
CHEMBL4282    & 3004       & 2112   & 892      & 33.81          & 37.69          \\
CHEMBL4722    & 2565       & 1750   & 815      & 31.93          & 35.29          \\ \bottomrule
\end{tabular}
\vspace{0.2cm}
\caption{Details of the ChEMBL datasets. Note that compounds with $pIC_{50} >= 6.3$ are considered as active molecules.}
\label{tab:dataset-details}
\end{table}


\paragraph{Experimental Setup:} For the ChEMBL datasets, we used $1$ layer of GIN with a hidden dimension of $64$ to limit the information gained from message passing. The final $3$-layer MLP had hidden dimension of $32$. We ran the experiment for $50$ epochs with an initial learning rate set to be $1e^{-2}$, halved every $10$ epochs. For experiments with additional fingerprints (See Table~\ref{tab:results}), we used a total of $25$ epochs with an initial learning rate of $1e^{-4}$, halved every $10$ epochs. For all the experiments, we sample every $10$th point in the $100\times 100$ grid resulting in a total of $100$ center points.

\end{document}